%% file: anonymous-submission-latex-2025.tex
\def\year{2022}\relax
\documentclass[letterpaper]{article} % DO NOT CHANGE THIS
\usepackage{aaai25}
\usepackage{mathrsfs}
\usepackage{times}  % DO NOT CHANGE THIS
\usepackage{helvet}  % DO NOT CHANGE THIS
\usepackage{courier}  % DO NOT CHANGE THIS
\usepackage[hyphens]{url}  % DO NOT CHANGE THIS
\usepackage{graphicx} % DO NOT CHANGE THIS
\usepackage{booktabs}
\usepackage{natbib}  % DO NOT CHANGE THIS AND DO NOT ADD ANY OPTIONS TO 
\usepackage{caption} % DO NOT CHANGE THIS AND DO NOT ADD ANY OPTIONS TO IT

\DeclareCaptionStyle{ruled}{labelfont=normalfont,labelsep=colon,strut=off} % DO NOT CHANGE THIS
\usepackage{multirow}
\usepackage[linesnumbered, ruled, vlined]{algorithm2e}
\usepackage{amsmath,amssymb,amsthm}
\usepackage{bm}
\newtheorem{assumption}{Assumption}
\newtheorem{lemma}{Lemma}
\newtheorem{theorem}{Theorem}
\newtheorem{remark}{Remark}

\usepackage{xcolor}

\usepackage{amsmath}
\usepackage{newfloat}
\usepackage{listings}
\usepackage{float}
\floatstyle{ruled}
\newfloat{listing}{tb}{lst}{}
\floatname{listing}{Listing}
\title{Conditional Latent Coding with Learnable Synthesized Reference for Deep Image Compression}
\author{
Siqi Wu\textsuperscript{\rm 1}\thanks{This work was completed during Siqi Wu's visiting research period at Southern University of Science and Technology.},
Yinda Chen\textsuperscript{\rm 2}\thanks{Co-first author.},
Dong Liu\textsuperscript{\rm 2},
Zhihai He\textsuperscript{\rm 3}\thanks{Corresponding author. Email: hezh@sustech.edu.cn.}
}
\affiliations{
\textsuperscript{\rm 1}University of Missouri, Columbia, MO, USA\\
\textsuperscript{\rm 2}University of Science and Technology of China, Hefei, China\\
\textsuperscript{\rm 3}Southern University of Science and Technology, Shenzhen, China
}

\usepackage{bibentry}
\begin{document}

\maketitle

\begin{abstract}
In this paper, we study how to synthesize a dynamic reference from an external dictionary to perform conditional coding of the input image in the latent domain and how to learn the conditional latent synthesis and coding modules in an end-to-end manner.
Our approach begins by constructing a universal image feature dictionary using a multi-stage approach involving modified spatial pyramid pooling, dimension reduction, and multi-scale feature clustering. For each input image, we learn to synthesize a conditioning latent by selecting and synthesizing relevant features from the dictionary, which significantly enhances the model's capability in capturing and exploring image source correlation. This conditional latent synthesis involves a correlation-based feature matching and alignment strategy, comprising a Conditional Latent Matching (CLM) module and a Conditional Latent Synthesis (CLS) module. The synthesized latent is then used to guide the encoding process, allowing for more efficient compression by exploiting the correlation between the input image and the reference dictionary. According to our theoretical analysis, the proposed conditional latent coding (CLC) method is robust to perturbations in the external dictionary samples and the selected conditioning latent, with an error bound that scales logarithmically with the dictionary size, ensuring stability even with large and diverse dictionaries. Experimental results on benchmark datasets show that our new method improves the coding performance by a large margin (up to 1.2 dB) with a very small overhead of approximately 0.5\%  bits per pixel. 
% Our code is publicly available at \url{https://github.com/ydchen0806/CLC}.
\end{abstract}
\begin{links}
    \link{Code}{https://github.com/ydchen0806/CLC}.
\end{links}
\section{Introduction}

With the rapid development of the Internet and mobile devices, billions of images are available in the world. For a given image, it is easy to find many correlated images on the Internet. It will be very interesting to explore how to utilize this vast amount of data to establish a highly efficient representation of the input image to improve the performance of deep image compression.
Continuous efforts have been made in the past two decades. The early attempt is to extract low-level feature patches from external images as a dictionary for image super-resolution \cite{sun2003resolution} and quality enhancement \cite{xiong2010robust}. Yue \textit{\textit{et al.}} \cite{yue2013cloud} proposed a cloud-based image coding scheme that utilizes a large-scale image database for reconstruction, achieving high compression ratios while maintaining visual quality. As data compression has shifted to the deep image/video compression paradigm in recent years, we would like to explore how to utilize the external dictionary of images to generate a dynamic reference representation to perform conditional coding of the input image within the deep image compression framework.

Deep neural network-based image compression methods \cite{balle2017end, toderici2017full, lee2019context} have made significant progress in recent years, surpassing traditional transform coding methods like JPEG in compression efficiency. However, current deep learning compression still faces challenges in efficiently exploring the source correlation of the image and maintaining high reconstruction quality at low bit rates. To further improve compression efficiency, researchers have begun to explore the use of external images as side information in distributed deep compression. For example, Ayzik \textit{et al.} \cite{ayzik2020deep} used auxiliary image information to perform block matching in the image domain, while Huang \textit{et al.} \cite{huang2023learned} extended this concept by introducing a multi-scale patch matching approach. However, this approach relies on specific auxiliary images, limiting its applicability and improvement.

% \begin{figure}[t]
%     \centering
%     \includegraphics[width=\linewidth]{figure/teasor2.pdf}
%     \caption{Overview of the proposed CLC: Key contributions (highlighted in red boxes) include conditional latent encoding/decoding and generating conditions from an external dictionary.}
%     \label{fig:teasor}
% \end{figure}

To overcome these limitations, we propose a novel framework called Conditional Latent Coding (CLC), which uses auxiliary information as a conditional probability at both the encoder and decoder. Our approach constructs a universal image feature dictionary using a multi-stage process involving modified spatial pyramid pooling (SPP), dimensionality reduction, and multi-scale feature clustering. For each input image, we generate a conditioning latent by adaptively selecting and learning to combine relevant features from the dictionary to generate a highly efficient reference representation, called \textit{conditioning latent}, for the input image. 
We then apply an advanced feature matching and alignment strategy, comprising a Conditional Latent Matching (CLM) module and a Conditional Latent Synthesis (CLS) module. This process leverages the conditioning latent to guide the encoding process, allowing for more efficient compression by exploiting similarities between the input image and the reference features. As demonstrated in our experimental results on benchmark datasets, our new method improves the coding performance by a large margin (up to 1.2 dB) at low bit-rates.

\section{Related Work and Unique Contributions} \label{sec:related_work}
Deep learning-based image compression has achieved remarked progress in recent years. Ballé \textit{et al.} \cite{balle2017end} pioneered an end-to-end optimizable architecture, later enhancing it with a hyperprior model \cite{balle2018variational} to improve entropy estimation. Transformer architectures have been proposed by Qian \textit{et al.} \cite{qian2022entroformer} to improve probability distribution estimation. Similarly, Cheng \textit{et al.} \cite{cheng2020learned} parameterizes the distributions of latent codes with discretized Gaussian Mixture models. Liu \textit{et al.} \cite{liu2023learned} combined CNNs and Transformers in the TCM block to explore the local and non-local source correlation. Yang \textit{et al.} \cite{yang2023tinc} proposed a Tree-structured Implicit Neural Compression (TINC) to maintain the continuity among regions and remove the local and non-local redundancy. To enhance the entropy coding performance, the conditional probability model and joint autoregressive and hierarchical priors model have been developed in \cite{mentzer2018conditional, minnen2018joint}. Jia \textit{et al.} \cite{jia2024generative} introduced a Generative Latent Coding (GLC) architecture to achieve high-realism and high-fidelity compression by transform coding in the latent space. 

This work is related to reference-based deep image compression, where reference information is used to improve coding efficiency. For example, Li \textit{et al.} \cite{li2021deep} pioneered this approach in video compression, while Ayzik \textit{et al.} \cite{ayzik2020deep} applied it at the decoder level. Sheng \textit{et al.} \cite{sheng2022temporal} proposed a temporal context mining module to propagate features and learn multi-scale temporal contexts. Huang \textit{\textit{et al.}} \cite{huang2023learned} extended the concept to multi-view image compression with advanced feature extraction and fusion. Li \textit{et al.} \cite{li2023neural} introduced the group-based offset diversity to explore the image context for better prediction. Zhao \textit{et al.} \cite{zhao2021universal} optimized the reference information using a universal rate-distortion optimization framework. \cite{zhao2023universal} integrated side information optimization with latent optimization to further enhance the compression ratio. In \cite{li2023rfd}, within the context of underwater image compression, a multi-scale feature dictionary was manually created to provide a reference for deep image compression based on feature matching. A content-aware reference frame selection method was developed in \cite{wu2022content} for deep video compression. 

\textbf{Unique contributions.} 
In comparison to existing methods, our work has the following unique contributions. (1) We develop a new approach, called conditional latent coding (CLC), which learns to synthesize a dynamic reference for each input image to achieve highly efficient conditional coding in the latent domain. 
(2) We develop a fast and efficient feature matching scheme based on ball tree search and an effective feature alignment strategy that dynamically balances compression bit-rate and reconstruction quality. (3) We developed a theoretical analysis to show that the proposed CLC method is robust to perturbations in the external dictionary samples and the selected conditioning latent, with an error bound that scales logarithmically with the dictionary size, ensuring stability even with large and diverse dictionaries.

\begin{figure}[tb]
    \centering
    \includegraphics[width=\linewidth]{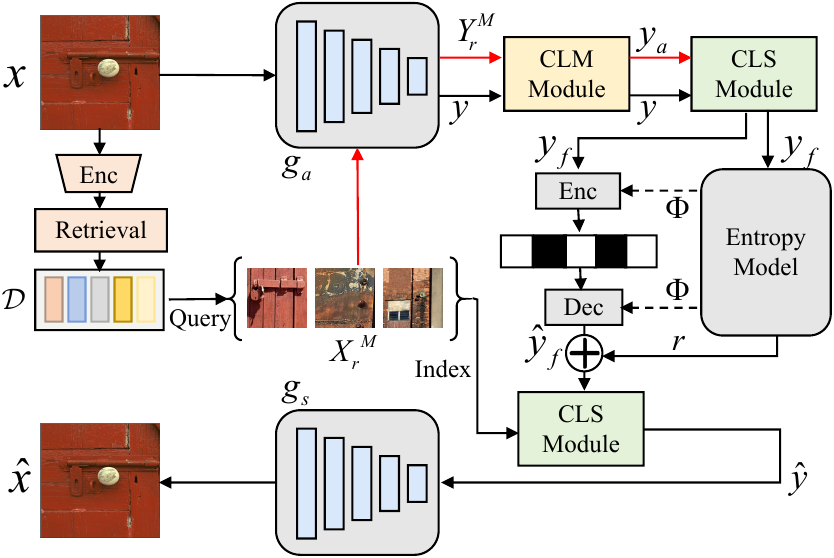}
    \caption{Overview of the proposed Conditional Latent Coding (CLC) framework.}
    %  \vspace{-0.2cm}
    \label{fig:main0810}
\end{figure}

\section{The Proposed CLC Method} \label{sec:method} 

\subsection{Method Overview}

% Given an input image $X \in \mathbb{R}^{H \times W \times 3}$ and a pre-trained feature reference dictionary $\mathcal{D} = \{\mathbf{d}_1, \mathbf{d}_2, ..., \mathbf{d}_K\}$, our method enhances image compression by leveraging dictionary information. We first extract features from $x$ to query $\mathcal{D}$, reconstructing a reference image $x_r$. Both $X$ and $x_r$ are processed through identical network paths, using an analysis transform $g_a$ to obtain latent representations $\mathbf{y}$ and $\mathbf{y}_r$, respectively. 

The overall architecture of our proposed CLC framework is illustrated in Figure~\ref{fig:main0810}. Given an input image $x$, we first construct a pre-trained feature reference dictionary $D$ from a large reference image dataset using a multi-stage approach involving feature extraction with modified spatial pyramid pooling (SPP), dimensionality reduction, and multi-scale feature clustering. Then, given an input image $x$, we extract its feature using an encoder $F_\theta$ which is used to query the dictionary $\mathcal{D}$ and find the top $M$ best-matching reference images $X_r^M=\{x_r^1, x_r^2, \cdots, x_r^M\}$. In this work, the default value of $M$ is 3. Both $x$ and the queried reference $X_r^M$ are passed through the encoder transform network $g_a$ to obtain their latent representations $y$ and $Y_r^M$, respectively. Using $Y_r^M$ as reference, we obtain $y_f$ through adaptive feature matching and multi-scale alignment and then learn a network to perform conditional latent coding of $y_f$. Simultaneously, a hyperprior network $h_a$ estimates a hyperprior $z$ from $y_f$ to provide additional context for entropy estimation.
A slice-based autoregressive context model is used for entropy coding, dividing $y_f$ into slices and using both $z$ and previously coded elements to estimate probabilities. During decoding, we first reconstruct $z$ and $y_f$ from the bitstream, then use the dictionary indices passed from the encoder to apply the same reference processing and alignment procedure to reconstruct $y$ from $y_f$, and finally reconstruct the image $\hat{x}$ using the synthesis transform $g_s$.
In the following section, we will explain the proposed CLC method in more detail.

% A hyperprior network $h_a$ further compresses these representations to generate $\mathbf{z}$ and $\mathbf{z}_r$. A patch matching module aligns $\mathbf{y}$ with $\mathbf{y}_r$, followed by an alignment module that fuses the matched features. We employ a slice-based autoregressive context model with attention mechanisms, dividing $\mathbf{y}$ into $K$ slices and using both previously processed slices and the aligned reference features to estimate the probability distribution of each slice. An improved KV-cache compression technique is applied to the attention mechanism to reduce memory usage. The estimated probabilities guide the entropy coding of $\mathbf{y}$ and $\mathbf{z}$, generating the final bitstream $b$. During decoding, we reconstruct $\hat{\mathbf{z}}$ and $\hat{\mathbf{y}}$ from $b$, and use the same reference image processing and alignment procedure to enhance the reconstruction. A synthesis transform $g_s$ then produces the final reconstructed image $\hat{x}$. The network is optimized end-to-end to minimize the rate-distortion function $L = D(X, \hat{x}) + \lambda(\mathbf{p}) R(b)$, where $\lambda(\mathbf{p})$ is an adaptive Lagrange multiplier. Our approach integrates aligned reference information at both encoder and decoder, leveraging rich prior knowledge to enhance compression efficiency and reconstruction quality.

\subsection{Constructing the Support Dictionary}
As stated in the above section, our main idea is to construct a universal feature dictionary from which a reference latent can be dynamically generated to perform conditional latent coding of each image.  Here, a critical challenge is constructing a universal feature dictionary that effectively represents diverse image content and enables efficient feature utilization throughout the compression pipeline. We address this challenge using a multi-stage approach that combines advanced feature extraction, dimensionality reduction, feature clustering, and fast and efficient dictionary access by the deep image compression system, as illustrated in Figure~\ref{fig:method1_framework}.

\begin{figure}[tb]
    \centering
    \includegraphics[width=\linewidth]{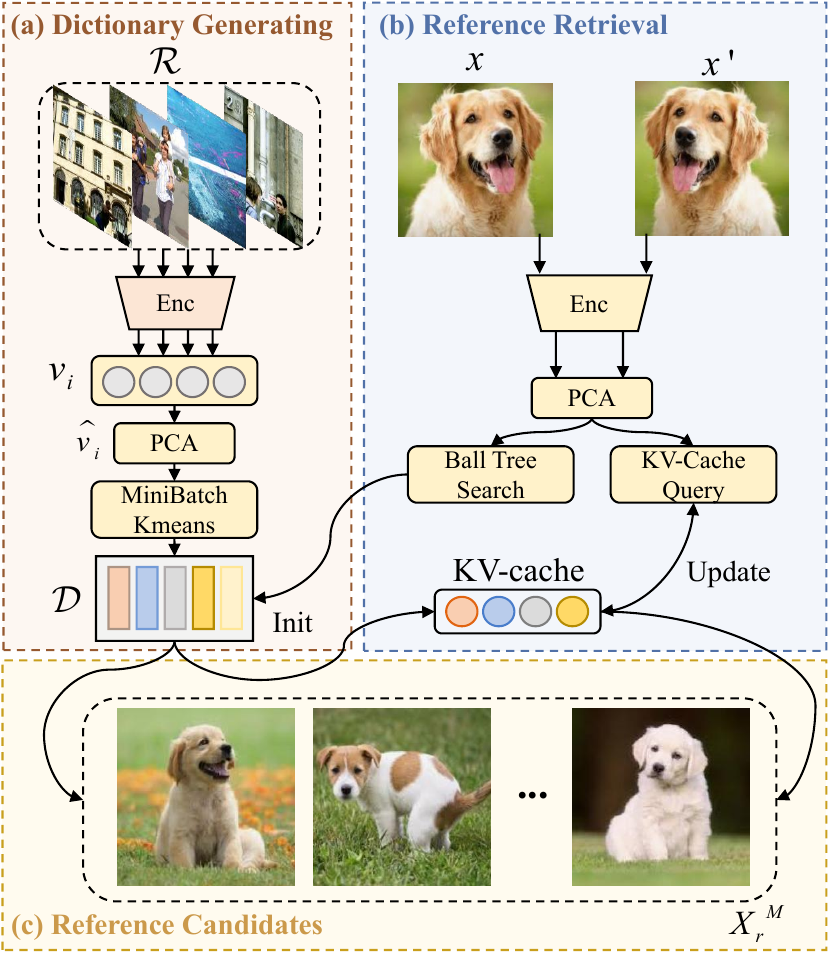}
    \caption{Universal Feature Dictionary Construction. (a) Dictionary Generation using diverse images $\mathcal{R}$ to create initial $\mathcal{D}$. (b) Reference Retrieval for querying and updating dictionary with inputs $x$ and $x'$. (c) Examples of reference candidates $X_r^M$ retrieved from the dictionary.}
    \label{fig:method1_framework}
\end{figure}

\paragraph{(1) Constructing the reference feature dictionary.} Our method begins with a large  reference dataset $\mathcal{R} = \{x_1, x_2, ..., x_N\}$. 
In this work, we randomly download 3000 images from the web. 
We use a modified pre-trained ResNet-50 model with Spatial Pyramid Pooling (SPP) as our feature extractor. For each image $x_i$, we extract its feature $\mathbf{v}_i = \text{SPP}(f_\theta(x_i))$, where $f_\theta(\cdot)$ represents the ResNet-50 backbone, and SPP aggregates features at scales $\{1\times1, 2\times2, 4\times4\}$. This multi-scale approach captures both global and local image characteristics.

To manage the high dimensionality of these features, we apply Principal Component Analysis (PCA), reducing each vector to 256 dimensions $\hat{\mathbf{v}}_i$. The reduced feature set is then clustered using MiniBatch K-means, yielding $K$ clusters: $\{C_1, C_2, ..., C_K\}$. From each cluster $C_j$, we select the feature vector closest to the centroid as its representative: $\mathbf{d}_j = \arg\min_{\hat{\mathbf{v}} \in C_j} \|\hat{\mathbf{v}} - \boldsymbol{\mu}_j\|_2$, where $\boldsymbol{\mu}_j$ is the centroid of $C_j$. These representatives form our feature dictionary $\mathcal{D} = \{\mathbf{d}_1, \mathbf{d}_2, ..., \mathbf{d}_K\}$.

\paragraph{(2) Fast and efficient dictionary matching.}
Our proposed CLC method deep image compression needs to access this dictionary during training and inference. 
One central challenge here is the dictionary search and matching efficiency. 
For efficient feature dictionary management and access, we introduce a KV-cache mechanism that is employed in both the initial feature retrieval and the subsequent encoding-decoding process. Specifically, we define our KV-cache as a tuple $(\mathbf{K}, \mathbf{V})$, where $\mathbf{K} \in \mathbb{R}^{N \times d_k}$ represents the keys and $\mathbf{V} \in \mathbb{R}^{N \times d_v}$ represents the values. Here, $N$ is the number of entries in the cache, $d_k$ is the dimension of the keys, and $d_v$ is the dimension of the values.

In the feature retrieval phase, we construct a ball tree over $\mathcal{D}$ for the initial coarse search, while maintaining the KV-cache. During compression, given an input image $x$, we extract its feature $f_\theta(x)$ and use it to query both the Ball Tree and the KV-cache. The retrieval process is formulated as a scaled dot-product attention mechanism:
\begin{equation}
    A(\mathbf{Q}, \mathbf{K}, \mathbf{V}) = \text{softmax}\left(\frac{\mathbf{Q}\mathbf{K}^T}{\sqrt{d_k}}\right)\mathbf{V},
\end{equation}
where $\mathbf{Q} = f_\theta(x)$, and $\mathbf{K}$ and $\mathbf{V}$ are the keys and values in the KV-cache, respectively.
To manage the size of the KV-cache and improve the matching efficiency, we implement a compression technique. Let $C: \mathbb{R}^{d} \rightarrow \mathbb{R}^{d'}$ be our compression function, where $d' < d$. We apply this to both keys and values:
\begin{equation}
    \mathbf{K}_c = C(\mathbf{K}), \quad \mathbf{V}_c = C(\mathbf{V}).
\end{equation}
The compression function $C$ is designed to preserve the most important information while reducing the dimensionality. In practice, we implement $C$ as a learnable neural network layer, optimized jointly with the rest of the system.
Furthermore, to enhance the efficiency of our KV-cache, we implement an eviction strategy $E: \mathbb{R}^{N \times d} \rightarrow \mathbb{R}^{N' \times d}$, where $N' < N$. This strategy removes less useful entries from the cache based on a relevance metric $\rho: \mathbb{R}^d \rightarrow \mathbb{R}$:
\begin{equation}
    (\mathbf{K}_e, \mathbf{V}_e) = E(\mathbf{K}, \mathbf{V}) = \text{TopK}(\rho(\mathbf{K}_i), \mathbf{K}, \mathbf{V}),
\end{equation}
where $\text{TopK}$ selects the top $K$ entries based on the relevance scores.
To further enhance robustness, we implement a multi-query strategy. For an input image $x$, we generate an augmented version $x'$ (e.g., by rotation) and perform separate queries for both. The final set of reference features is obtained by merging and de-duplicating the results.

% This comprehensive approach to feature dictionary construction and utilization enables our compression system to leverage rich reference information efficiently throughout the entire compression pipeline. By integrating the KV-cache mechanism with our feature dictionary and Ball Tree structure, we create a hybrid system that maintains high performance while dealing with the computational challenges of processing and utilizing large amounts of feature information. This approach potentially leads to improved compression performance across a diverse range of input images.

\subsection{Conditional Latent Synthesis and Coding}
As the unique contribution of this work, instead of simply finding the best match in existing methods \cite{jia2024generative}, the reference or side information for each image is dynamically generated in the latent domain by a learned network to best represent the input image.
Our method is motivated by the following observation: the central challenge in reference-based image compression is the large deviation between the arbitrary input image and the fixed and limited set of reference images. Our method finds multiple closest reference images and dynamically fuses them to form a best approximation of the input image in the latent domain. Specifically, the proposed conditional latent synthesis and coding method has the following major components:

\begin{figure}[t]
    \centering
    \includegraphics[width=\linewidth]{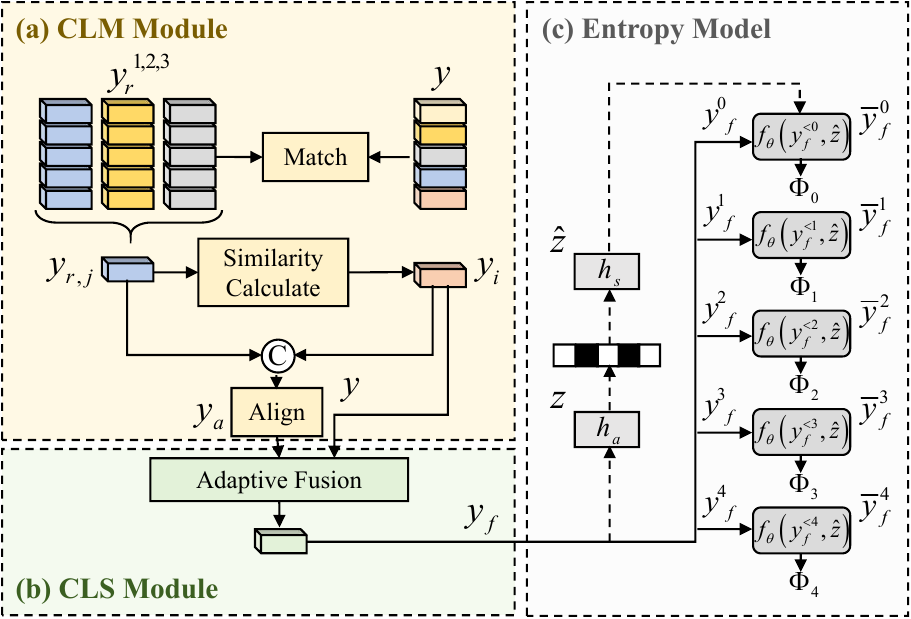}
    \caption{The detail of our proposed CLM and CLS module.}
    \label{fig:enter-label}
\end{figure}

\paragraph{(1) Feature Matching and Alignment.} 
We first propose an advanced feature matching and alignment scheme that aligns reference features from the dictionary with the input image. Our approach begins with a Conditional Latent Matching (CLM) module. Given an input image $x \in \mathbb{R}^{H \times W \times 3}$ and a pre-built feature reference dictionary $\mathcal{D} = \{\mathbf{d}_1, \mathbf{d}_2, ..., \mathbf{d}_K\}$, we first extract features from $x$ to query $\mathcal{D}$, retrieving the top $M$ feature and their corresponding reference images $X_r^M=\{x_r^1, x_r^2, \cdots, x_r^M\}$. Both $x$ and $X_r^M$ are then processed through the same analysis transform network. In this work, we use the Transformer-CNN Mixture (TCM) block \cite{liu2023learned}, which efficiently combines the strengths of CNNs for local feature extraction and transformers for capturing long-range dependencies. TCM blocks are used at both encoder $g_a$, decoder $g_s$, and hyperprior network $h_a$, enabling effective feature processing at various stages of the compression pipeline.

The analysis transform $g_a$ converts $x$ and $x_r^M$ into latent representations $y$ and $Y_r^M$, respectively. The CLM then establishes correspondences between $Y_r^M$ and $y$, addressing the issues of spatial inconsistencies. It computes $y_m = \mathcal{F}_m(y, Y_r^M; \theta_m)$, where $\mathcal{F}_m$ is a learnable function parameterized by $\theta_m$. This function computes a similarity matrix $\mathbf{S}$ between features of $y$ and $y_r$:
\begin{equation}
    S_{ij} = \frac{\exp(\langle \phi(y_i), \phi(y_{r,j}) \rangle / \tau)}{\sum_k \exp(\langle \phi(y_i), \phi(y_{r,k}) \rangle / \tau)},
\end{equation}
where $\phi(\cdot)$ is a learnable feature transformation that maps input features to a higher-dimensional space, $\langle \cdot, \cdot \rangle$ denotes inner product, and $\tau$ is a temperature parameter. 
We also introduce a learnable alignment module within the CLM to refine the alignment between reference and target features: $y_a = \mathcal{F}_a(y, y_m; \theta_a)$, where $\mathcal{F}_a$ is implemented as a series of deformable convolution layers operating at multiple scales.

\paragraph{(2) Conditional Latent Synthesis.} In the final stage of our feature matching and alignment strategy, we develop a Conditional Latent Synthesis (CLS) module to fuse the aligned reference features with the target image feature. We model this fusion process as a conditional probability with learnable weights:
\begin{equation}
    p(y_f | y, y_a) = \mathcal{N}(\mu(y, y_a), \sigma^2(y, y_a)),
\end{equation}
where $y_f$ is the final latent representation, and $\mu(\cdot)$ and $\sigma^2(\cdot)$ are learnable functions implemented as neural networks. These functions estimate the mean and variance of the Gaussian distribution for $y_f$ conditioned on both $y$ and $y_a$. The mean function $\mu(\cdot)$ is designed to incorporate adaptive weighting:
\begin{equation}
    \mu(y, y_a) = \alpha \odot y + (1 - \alpha) \odot y_a,
\end{equation}
where $\alpha$ are dynamically computed weights based on content: $\alpha = \sigma(\mathcal{F}_w([y, y_a]; \theta_f))$. Here, $\sigma$ is the sigmoid function, and $\mathcal{F}_w$ is a small neural network predicting optimal fusion weights. This conditional generation approach with adaptive weights allows our model to capture complex dependencies between the input image and the reference image from the dictionary in the latent space, resulting in more flexible and powerful conditional coding.
During training, we sample from this distribution to obtain $y_f$, while during inference, we use the mean $\mu(y, y_a)$ as the final latent representation. This probabilistic formulation enables our model to handle uncertainties in the feature integration process and potentially generate diverse latent representations during training, which can improve the robustness and generalization capability of our deep compression system.

\paragraph{(3) Entropy Coding and Hyperprior.} To further improve compression efficiency, we introduce a hyperprior network $h_a$ that estimates a hyperprior $z$ from the conditional latent $y_f = h_a(y_f)$. This hyperprior $z$ provides additional context for more accurate probability estimation of $y_f$, enhancing the entropy model. The hyperprior is quantized and encoded separately, $\hat{z} = Q(z)$, where $Q(\cdot)$ denotes the quantization operation.

For entropy coding, we adopt a slice-based auto-regressive context model \cite{}. The conditional representation $y_f$ is divided into $K$ slices: $y_f = [y_f^1, y_f^2, ..., y_f^K]$. The probability distribution of each slice is estimated using both previously processed slices and the hyperprior information. For the $i$-th slice, the probability model is expressed as:
\begin{equation}
    p(y_f^i|y_f^{<i}, \hat{z}) = f_\theta(y_f^{<i}, \hat{z}),
\end{equation}
where $f_\theta$ is a neural network parameterized by $\theta$, and $y_f^{<i} = [y_f^1, ..., y_f^{i-1}]$ represents all previously encoded slices. 
The output of $f_\theta$ is used to parametrize a probability distribution. Specifically, we model each element of $y_f^i$ as a Gaussian distribution with mean $\mu_i$ and scale $\sigma_i$:
\begin{equation}
    p(y_f^i|y_f^{<i}, \hat{z}) \sim \mathcal{N}(\mu_i, \sigma_i^2),
\end{equation}
where $\Phi_i = (\mu_i, \sigma_i) = f_\theta(y_f^{<i}, \hat{z})$. Here, $\Phi_i$ represents the distribution parameters for the $i$-th slice.
This approach captures complex dependencies within the latent representation, leading to more efficient compression.
During the entropy coding process, we compute a residual $r_i$ for each slice: $r_i = y_f^i - \hat{y}_f^i$, where $\hat{y}_f^i$ is the quantized version of $y_f^i$. This residual helps to reduce quantization errors and improve reconstruction quality.
The actual encoding process involves quantizing $y_f^i - \mu_i$ and entropy encoding the result using the estimated distribution $\mathcal{N}(0, \sigma_i^2)$. During decoding, we reconstruct $\hat{y}_f^i$ as $\hat{y}_f^i = Q(y_f^i - \mu_i) + \mu_i$, where $Q(\cdot)$ denotes the quantization operation.

\paragraph{(4) Decoding and Optimization.} During decoding, we first reconstruct $\hat{z}$ and $\hat{y}_f$ from the bitstream. Then, using the dictionary indices passed from the encoder, we apply the same reference processing and alignment procedure to reconstruct $y$ from $\hat{y}_f$. Next, $y$ is fed into the synthesis transform $g_s$ to produce the final reconstructed image $\hat{x}$. It is important that we employ the same conditional latent synthesis pipeline on the decoder side to ensure consistency.
The combination of the hyperprior $z$ and the slice-based autoregressive model enables our system to achieve a fine balance between capturing global image statistics and local, contextual information, resulting in improved compression performance.
To optimize our network end-to-end, we minimize the rate-distortion function:
\begin{equation}
L = D(x, \hat{x}) + \lambda R(b),
\end{equation}
where $D(x, \hat{x})$ is the distortion between the original and reconstructed images, $R(b)$ is the bitrate of the encoded stream, and $\lambda$ is an adaptive coefficient used to balance the rate-distortion trade-off. This optimization balances compression efficiency and reconstruction quality, allowing our approach to effectively leverage the aligned reference information at both the encoder and decoder stages.

\subsection{Theoretical Perturbation Analysis} \label{sec:analysis}
In image compression with auxiliary information, some degree of error in feature retrieval is inevitable due to the inherent complexity of the problem and the presence of noise. Understanding the bounds of this error is crucial for assessing and improving compression algorithms. We present a theoretical framework that quantifies these errors and provides insights into the factors affecting compression performance.

We formulate the problem as a rate-distortion optimization:
\begin{equation}
\min_{G_1, G_2, D} \mathbb{E}\left[R(G_1(x), G_2(\tilde{x})) + \lambda \mathscr{D}\left(x, D(G_1(x), G_2(\tilde{x}))\right)\right]\nonumber
\end{equation}
where $x \in \mathbb{R}^d$ is the original image, $\tilde{x} \in \mathbb{R}^d$ the auxiliary image, $G_1$ and $G_2$ are encoders, $D$ is a decoder, $R$ is the rate loss, and $\mathscr{D}$ is the distortion loss.

Our analysis is based on several key assumptions. We model the original image using a spiked covariance model: $x = U^s + \xi$, and the auxiliary image similarly: $\tilde{x} = U^*\tilde{s} + \tilde{\xi}$. The rate loss is entropy-based: $R(z, \tilde{z}) = \mathbb{E}[-\log_2 p_\theta(z|\tilde{z})]$, while the distortion loss is mean squared error: $\mathscr{D}(x, \hat{x}) = \|x - \hat{x}\|^2$. We assume sub-Gaussian noise with parameter $\sigma^2$, and allow for possible irrelevant information in the auxiliary image, with proportion $p \in [0, 1)$.

Our theoretical analysis aims to quantify the error in feature retrieval when using auxiliary information for image compression, specifically establishing an upper bound on the error in estimating the feature subspace of the original image, with a focus on the impact of irrelevant information in the auxiliary image. This analysis provides a rigorous foundation for understanding our Conditional Latent Coding (CLC) method, quantifies trade-offs between factors affecting compression performance, and offers insights into the method's robustness to imperfect auxiliary data. By emphasizing the importance of minimizing irrelevant information, it guides the design and optimization of our dictionary construction process. By deriving this error bound, we bridge the gap between theoretical understanding and practical implementation, providing a solid basis for the development and refinement of our compression algorithm.

Our main result quantifies the unavoidable error in feature retrieval:

\begin{theorem}
For any $\delta > 0$, with probability at least $1-\delta$:
\begin{equation}
\resizebox{.9\hsize}{!}{$
\|\sin \Theta(\mathrm{Pr}(\hat{G}_1), U^*)\|_F \leq C\left(\sqrt{r} \wedge \sqrt{\frac{r}{1-p}} \sqrt{\frac{(r + r(\Sigma_\xi))\log(d/\delta)}{n}}\right)
$}\nonumber
\end{equation}

where $C > 0$ is a constant, $p$ is the proportion of irrelevant parts in the auxiliary image, $n$ is the number of training samples, $r(\Sigma_\xi)$ is the effective rank of the noise covariance matrix, and $\hat{G}_1$ is the estimated encoder for the original image.
\end{theorem}

This bound provides key insights: it reveals a trade-off between problem dimensionality ($r$), sample size ($n$), noise structure ($r(\Sigma_\xi)$), and auxiliary image quality ($p$). The system's tolerance to irrelevant information is quantified by $\frac{1}{1-p}$, while noise complexity is captured by the effective rank $r(\Sigma_\xi)$. The result also suggests potential for mitigation through increased sample size or improved auxiliary image quality.

% This analysis not only acknowledges the inherent limitations in feature retrieval but also provides a foundation for understanding and potentially improving compression performance in practical applications, especially in scenarios with imperfect or noisy auxiliary data.

% A complete problem description, detailed proofs, and robustness experiments are provided in the \textbf{Supplementary Materials}.

\section{Experimental Results} \label{sec:experiments}
\begin{figure*}[t]
\centering
\includegraphics[width=0.9\linewidth]{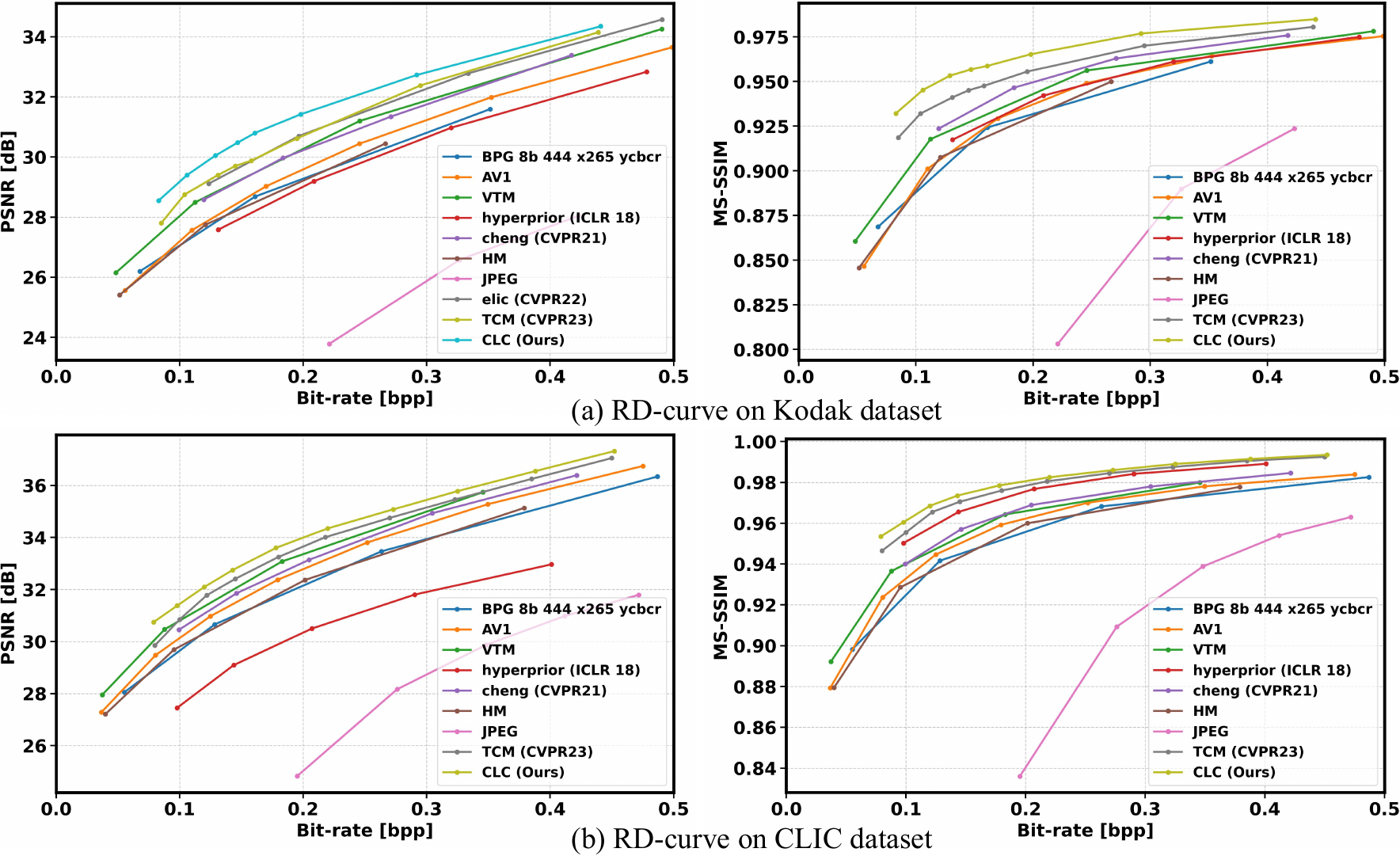}
 \vspace{-0.2cm}
\caption{The rate-distortion performance comparison of different methods.}
 \vspace{-0.2cm}
\label{fig:rate_distortion}
\end{figure*}
In this section, we provide extensive experimental results to evaluate the proposed CLC method and ablation studies to understand its performance.

\subsection{Experimental Settings}

\subsubsection{(1) Datasets.}
In our experiments, we use two benchmark datasets: Flickr2W \cite{liu2020unified} and Flickr2K \cite{timofte2017ntire}. The Flickr2W dataset, containing 20,745 high-quality images, was used for training our model. To construct the image feature dictionary, we employed the Flickr2K dataset, which comprises 2,650 images. These Flickr2K images were randomly cropped into 256×256 patches to build the feature reference dictionary. We evaluated our algorithm on the Kodak \cite{kodak1993suite} and CLIC \cite{toderici2020clic} datasets to evaluate its performance.

%For performance evaluation, we adopted two widely used metrics in image coding: Peak Signal-to-Noise Ratio (PSNR) and Multi-Scale Structural Similarity Index (MS-SSIM). PSNR provides a quantitative measure of the reconstructed image quality, while MS-SSIM offers a perceptual quality assessment that correlates well with human visual perception. These metrics were chosen for their complementary nature in assessing both objective and perceptual image quality.

\subsubsection{(2) Implementation Details.}
Our model was implemented using PyTorch and trained on 8 NVIDIA RTX 3090 GPUs. We trained the network for 30 epochs using the Adam optimizer with an initial learning rate of $1\times10^{-4}$, which was reduced by a factor of 0.5 every 10 epochs. The batch size was set to 16 for each GPU.
For the patch matching module, we used a patch size of 16$\times$16 pixels. The initial value of the adaptive fusion weight $\alpha$ was set to 0.5. The number of slices K in the slice-based autoregressive context model was set to 8.
In the KV-cache, we set the dimension of keys $d_k$ and values $d_v$ to 256. The cache size $N$ was initially set to 300 and dynamically adjusted based on the GPU memory availability. The number of clusters $K$ in Mini-Batch K-means was set to 3,000.

%We compared our method with standard codecs including VVC (VTM 12.0), HEVC (HM 16.20), and JPEG (libjpeg 9d), using their default configurations. For learned compression methods, we used the official implementations with their default parameters.

\subsection{Performance Results}

We report the rate-distortion results in Figure~\ref{fig:rate_distortion}, showing our proposed CLC method outperforms existing methods across different bit-rates. The compared methods include traditional codecs like BPG~\cite{bellard2014bpg}, VTM~\cite{bross2021overview}, HM~\cite{sullivan2012overview}, and JPEG~\cite{wallace1992jpeg}, as well as recent learning-based methods: the hyperprior model~\cite{balle2018variational}, Cheng et al.'s approach~\cite{cheng2021learned}, ELIC~\cite{zou2022elic}, and TCM~\cite{chen2023transformer}. We also include results from AV1~\cite{chen2018overview} for comparison. The improvement in compression efficiency is significant. On Kodak at MS-SSIM 0.95, CLC achieves 0.1 bpp, while TCM, VTM, BPG, and JPEG require 0.15, 0.18, 0.22, and 0.38 bpp, respectively, representing a 1.5 to 3.8 times increase in compression ratio. On CLIC at 34 dB PSNR, CLC achieves 0.2 bpp, compared to 0.25, 0.28, 0.35, and 0.45 bpp for TCM, VTM, Hyperprior, and JPEG, indicating larger efficiency gains. Figure~\ref{fig:main_visaul} demonstrates our method's superior performance in preserving detailed textures, particularly horizontal and vertical structures at low bit rates, as seen in railings and architectural features.

\subsection{Ablation Studies}

We conducted ablation studies to evaluate components of our CLC method, focusing on reference images, dictionary cluster size, and component contributions. We report results on both Kodak and CLIC datasets to demonstrate the performance across different image types.

\paragraph{(1) Ablation Studies on the Number of Reference Images.}
% We changed the number of reference images from 1 to 5 to examine its impact on the compression performance. Table \ref{tab:ref_images} shows the BD-rate savings compared to the VTM method with different numbers of reference images. Here, BD-Rate$_\text{P}$ represents the BD-rate savings in terms of PSNR, while BD-Rate$_\text{M}$ represents the BD-rate savings in terms of MS-SSIM.
% We can see that using three reference images achieves the best performance on both datasets. Compared to the VTM methods, it saves the BD-rate by 14.5\% and 13.9\% on the Kodak and CLIC datasets, respectively. When the number of reference images becomes larger than 3, the performance degrades. This is because too much redundancy has been introduced into the conditional coding process.

We changed the number of reference images from 1 to 5 to examine the impact on compression performance. Table \ref{tab:ref_images} shows BD-rate savings compared to the VTM method with different numbers of reference images. BD-Rate$_\text{P}$ represents savings in PSNR, while BD-Rate$_\text{M}$ represents savings in MS-SSIM. Using three reference images achieves the best performance on both datasets, saving 14.5\% and 13.9\% BD-rate on Kodak and CLIC, respectively. More than three images introduce redundancy, degrading performance.

\begin{table}[t]
%  \vspace{-0.2cm}

\centering
\small
\setlength{\tabcolsep}{4pt}
\renewcommand{\arraystretch}{0.85} % 调整行间距
\begin{tabular}{@{}lcccc@{}}
\toprule[1.5pt]
\multirow{2}{*}{\begin{tabular}[c]{@{}l@{}}Num of\\Ref. Images\end{tabular}} & \multicolumn{2}{c}{Kodak} & \multicolumn{2}{c}{CLIC} \\
\cmidrule(lr){2-3} \cmidrule(lr){4-5}
& BD-Rate$_\text{P}$ & BD-Rate$_\text{M}$ & BD-Rate$_\text{P}$ & BD-Rate$_\text{M}$ \\
\midrule
1 & -10.2 & -11.5 & -9.8 & -10.9 \\
2 & -12.8 & -13.7 & -12.1 & -13.2 \\
3 & \textbf{-14.5} & \textbf{-15.2} & \textbf{-13.9} & \textbf{-14.7} \\
4 & -14.3 & -15.0 & -13.7 & -14.5 \\
5 & -14.2 & -14.9 & -13.6 & -14.4 \\
\bottomrule[1.5pt]
\end{tabular}
\caption{BD-rate savings (\%) vs. VTM for different numbers of reference images.}
\label{tab:ref_images}
%  \vspace{-0.2cm}
\end{table}

\begin{figure*}[t]
    \centering
    \includegraphics[width=0.85\linewidth]{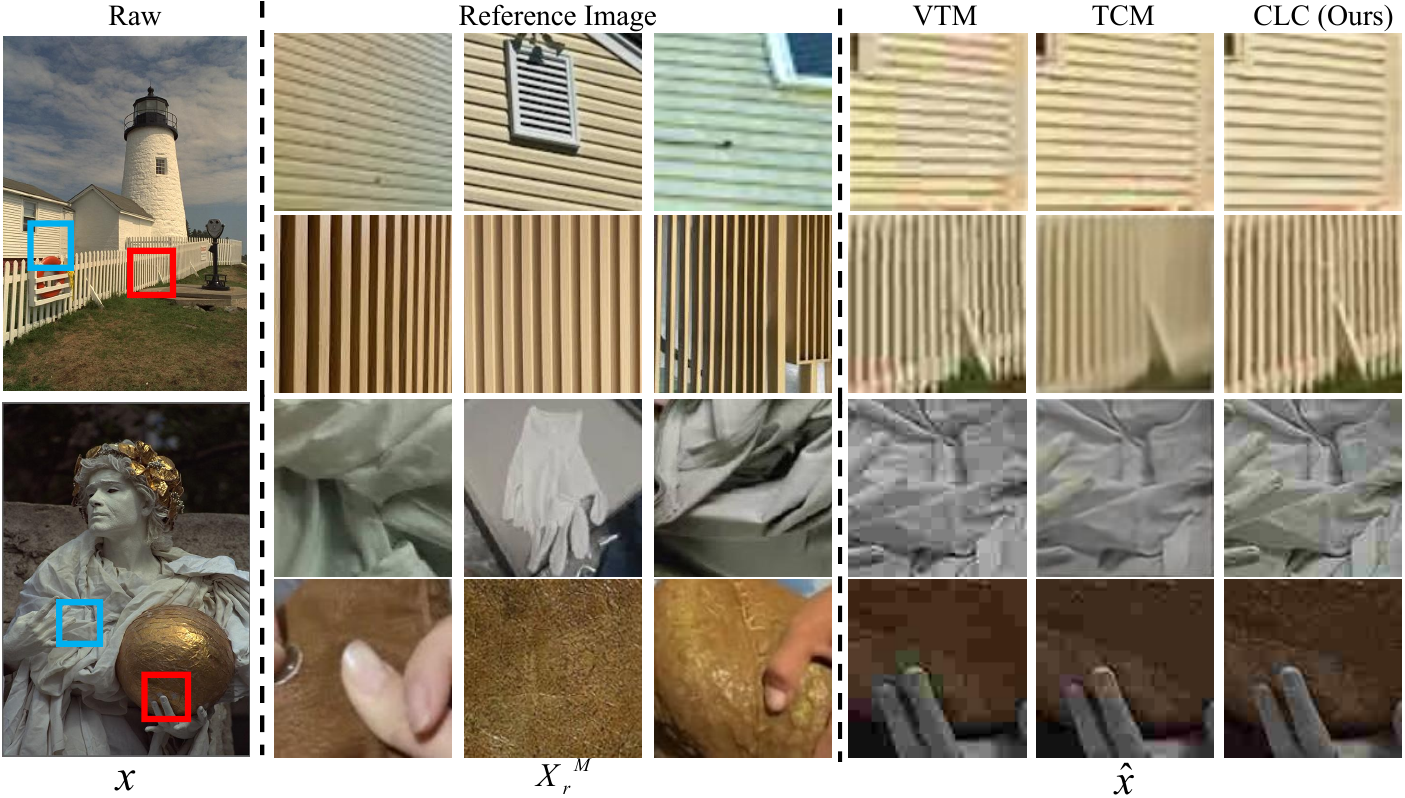}
     \vspace{-0.2cm}
    \caption{Image reconstruction results at around 0.1 bpp. From left to right: Raw inputs, reference images, reconstructed images. Red and blue boxes highlight specific areas of improvement.}
    \label{fig:main_visaul}
\end{figure*}
\paragraph{(2) Ablation Studies on the Dictionary Cluster Size.}
We conducted experiments with different dictionary cluster sizes to find the balance between compression efficiency and computational complexity. Table \ref{tab:dict_size} shows the BD-rate savings and encoding time for different cluster sizes. A cluster size of 3000 provides the best trade-off between performance and complexity for both datasets, achieving significant BD-rate savings with reasonable encoding times. The sharp increase in the encoding time for cluster sizes beyond 3000 highlights the importance of carefully selecting this parameter to balance compression efficiency and computational cost.

\begin{table}[t]

%  \vspace{-0.2cm}

\centering
\small
\setlength{\tabcolsep}{2pt} % 调整列间距
\renewcommand{\arraystretch}{0.85} % 调整行间距
\begin{tabular}{@{}lccccc@{}}
\toprule[1.5pt]
\multirow{2}{*}{\begin{tabular}[c]{@{}l@{}}Cluster\\Size\end{tabular}} & \multicolumn{2}{c}{Kodak} & \multicolumn{2}{c}{CLIC} & \multirow{2}{*}{\begin{tabular}[c]{@{}c@{}}Encoding\\Time (s)\end{tabular}} \\
\cmidrule(lr){2-3} \cmidrule(lr){4-5}
& BD-Rate$_\text{P}$ & BD-Rate$_\text{M}$ & BD-Rate$_\text{P}$ & BD-Rate$_\text{M}$ & \\
\midrule
1000 & -11.8 & -12.5 & -11.2 & -11.9 & 0.52 \\
2000 & -13.7 & -14.3 & -13.1 & -13.8 & 0.78 \\
3000 & \textbf{-14.5} & \textbf{-15.2} & \textbf{-13.9} & \textbf{-14.7} & 1.05 \\
4000 & -14.6 & -15.3 & -14.0 & -14.8 & 2.31 \\
5000 & -14.7 & -15.4 & -14.1 & -14.9 & 5.67 \\
\bottomrule[1.5pt]
\end{tabular}
%  \vspace{-0.2cm}
\caption{BD-rate savings (\%) vs. VTM and encoding time for different dictionary cluster sizes}
\label{tab:dict_size}
\end{table}

\paragraph{(3) Ablation Studies on Major Algorithm Components.}
We conducted ablation experiments to evaluate the contribution of major components. Table \ref{tab:components} shows each component's impact on the Kodak dataset performance. All components contribute significantly, with CLS having the most substantial impact (4.7\% BD-rate savings), highlighting the importance of adaptive feature modulation. The KV-cache, while minimally impacting compression performance, significantly reduces encoding time (from 1.87s to 1.05s). Multi-sample query in dictionary construction improves BD-rate savings by 0.7\% (BD-Rate$_\text{P}$) and 0.6\% (BD-Rate$_\text{M}$), enhancing overall compression capability through more diverse representations.

\begin{table}[t]

%  \vspace{-0.2cm}

\centering
\small
\setlength{\tabcolsep}{3pt}
\renewcommand{\arraystretch}{0.85} % 调整行间距
\begin{tabular}{@{}lccc@{}}
\toprule[1.5pt]
\multirow{2}{*}{\begin{tabular}[c]{@{}l@{}}Model\\Configuration\end{tabular}} & \multicolumn{2}{c}{BD-Rate Savings} & \multirow{2}{*}{\begin{tabular}[c]{@{}c@{}}Encoding\\Time (s)\end{tabular}} \\
\cmidrule(lr){2-3}
& BD-Rate$_\text{P}$ & BD-Rate$_\text{M}$ & \\
\midrule
Full Model & \textbf{-14.5} & \textbf{-15.2} & 1.05 \\
w/o CLM & -12.3 & -13.1 & 0.98 \\
w/o CLS & -9.8 & -10.5 & 0.92 \\
w/o KV-cache & -14.4 & -15.1 & 1.87 \\
w/o Multi-example Query & -13.8 & -14.6 & 0.97 \\
\bottomrule[1.5pt]
\end{tabular}
%  \vspace{-0.2cm}
\caption{BD-rate savings (\%) vs. VTM for different model configurations on Kodak dataset}
\label{tab:components}
\end{table}

\section{Conclusion} \label{sec:conclusion}
This study proposes Conditional Latent Coding (CLC), a novel deep learning-based image compression method that dynamically generates latent reference representations through a universal image feature dictionary. We develop innovative techniques for dictionary construction, efficient search/matching, alignment, and fusion, with theoretical analysis of robustness to dictionary and latent perturbations. While focused on compression, CLC's adaptive feature utilization principles may inspire broader vision tasks. Future work includes balancing compression efficiency and visual information utilization to address growing data transmission demands.

\clearpage
\section*{Acknowledgements}
This research was supported by the National Natural Science Foundation of China (No. 62331014) and Grant 2021JC02X103.
% \bibliographystyle{aaai25.bst}
% \bibliography{aaai25}

\input{anonymous-submission-latex-2025.bbl}
\appendix
\input{figure/appendix_after}

\end{document}

%% file: figure/appendix_after.tex
\section{Theoretical Proof}

\subsection{Problem Formulation}

Let $x \in \mathbb{R}^d$ be the original image, and $\tilde{x} \in \mathbb{R}^d$ be the reference image used for side information. We define encoders $G_1: \mathbb{R}^d \to \mathbb{R}^r$ and $G_2: \mathbb{R}^d \to \mathbb{R}^r$ for the original and reference images respectively, and a decoder $D: \mathbb{R}^r \times \mathbb{R}^r \to \mathbb{R}^d$.

The rate-distortion optimization problem is formulated as:
\begin{equation}\label{eq:RD_optimization}
\min_{G_1, G_2, D} \ \mathbb{E}_{x, \tilde{x}} \left[ R\left( G_1(x), G_2(\tilde{x}) \right) + \lambda \cdot D\left( x, D\left( G_1(x), G_2(\tilde{x}) \right) \right) \right],
\end{equation}
where $R(\cdot, \cdot)$ is the rate (compression) loss, $D(\cdot, \cdot)$ is the distortion loss (e.g., reconstruction error), and $\lambda > 0$ is a weighting parameter balancing rate and distortion.

\subsection{Assumptions}

\begin{assumption}\label{assump:spiked_covariance}
\textbf{Spiked Covariance Model for Images:}

The original image $x$ follows a spiked covariance model:
\begin{equation}\label{eq:spiked_model_original}
x = U^* s + \xi,
\end{equation}
where:
\begin{itemize}
    \item $U^* \in \mathbb{R}^{d \times r}$ is the true low-rank feature matrix with orthonormal columns ($U^{*T} U^* = I_r$).
    \item $s \in \mathbb{R}^r$ is the latent representation, with $\mathbb{E}[s] = 0$ and $\mathbb{E}[s s^T] = \Sigma_s$.
    \item $\xi \in \mathbb{R}^d$ is additive noise, independent of $s$, with zero mean and covariance $\Sigma_{\xi} = \sigma_{\xi}^2 I_d$.
\end{itemize}
\end{assumption}

\begin{assumption}\label{assump:reference_image}
\textbf{Reference Image with Irrelevant Parts:}

The reference image $\tilde{x}$ is given by:
\begin{equation}\label{eq:reference_image_model}
\tilde{x} = U^* (\rho s + \sqrt{1 - \rho^2} s_\perp) + \tilde{\xi},
\end{equation}
where:
\begin{itemize}
    \item $\rho \in [0,1]$ represents the correlation between $x$ and $\tilde{x}$.
    \item $s_\perp \in \mathbb{R}^r$ is independent of $s$, with $\mathbb{E}[s_\perp] = 0$ and $\mathbb{E}[s_\perp s_\perp^T] = \Sigma_s$.
    \item $\tilde{\xi} \in \mathbb{R}^d$ is additive noise, independent of $s$ and $s_\perp$, with zero mean and covariance $\Sigma_{\tilde{\xi}} = \sigma_{\tilde{\xi}}^2 I_d$.
    \item The total irrelevant proportion in $\tilde{x}$ is characterized by $p = 1 - \rho^2$.
\end{itemize}
\end{assumption}

\begin{assumption}\label{assump:entropy_model}
\textbf{Entropy Model for Rate Loss:}

The rate loss is based on a Gaussian entropy model:
\begin{equation}\label{eq:rate_loss}
R(z, \tilde{z}) = \mathbb{E}_{z, \tilde{z}} \left[ - \log_2 p_\theta( z \mid \tilde{z} ) \right],
\end{equation}
where $p_\theta( z \mid \tilde{z} )$ is a conditional Gaussian distribution:
\begin{equation}\label{eq:conditional_gaussian}
p_\theta( z \mid \tilde{z} ) = \mathcal{N}\left( z; \mu( \tilde{z} ), \Sigma_z \right),
\end{equation}
with $\mu( \tilde{z} )$ and $\Sigma_z$ being the mean and covariance conditioned on $\tilde{z}$.
\end{assumption}

\begin{assumption}\label{assump:distortion_loss}
\textbf{Distortion Loss:}

The distortion loss is defined as the mean squared error between the original image and the reconstructed image:
\begin{equation}\label{eq:distortion_loss}
D( x, \hat{x} ) = \| x - \hat{x} \|_2^2,
\end{equation}
where $\hat{x} = D( G_1(x), G_2(\tilde{x}) )$.
\end{assumption}

\begin{assumption}\label{assump:subgaussian_noise}
\textbf{Sub-Gaussian Noise:}

The noise vectors $\xi$ and $\tilde{\xi}$ are sub-Gaussian with parameter $\sigma^2$, i.e., for any $u \in \mathbb{R}^d$ with $\| u \|_2 = 1$,
\begin{equation}
\mathbb{P}\left( | u^T \xi | \geq t \right) \leq 2 \exp\left( - \frac{ t^2 }{ 2 \sigma^2 } \right), \quad \forall t > 0.
\end{equation}
\end{assumption}

\subsection{Main Results}

\begin{lemma}\label{lem:conditional_entropy}
Under Assumptions \ref{assump:spiked_covariance}--\ref{assump:entropy_model}, the rate loss $R(z, \tilde{z})$ can be expressed as:
\begin{equation}\label{eq:rate_loss_expression}
R(z, \tilde{z}) = \frac{1}{2 \ln 2} \left( r \ln (2 \pi e) + \ln \det( \Sigma_z ) \right).
\end{equation}
\end{lemma}

\begin{proof}
Since $p_\theta( z \mid \tilde{z} )$ is a Gaussian distribution, the differential entropy is:
\begin{equation}
h( z \mid \tilde{z} ) = \frac{1}{2} \ln \left( (2 \pi e)^r \det( \Sigma_z ) \right).
\end{equation}
Thus, the rate loss is:
\begin{equation}
R(z, \tilde{z}) = - \mathbb{E}_{z, \tilde{z}} \left[ \log_2 p_\theta( z \mid \tilde{z} ) \right] = \frac{1}{\ln 2} h( z \mid \tilde{z} ),
\end{equation}
which leads to Equation (\ref{eq:rate_loss_expression}).
\end{proof}

\begin{theorem}\label{thm:recovery_error_bound}
Under Assumptions \ref{assump:spiked_covariance}--\ref{assump:subgaussian_noise}, let $\hat{G}_1$ be the estimated encoder for the original image obtained from solving the optimization problem (\ref{eq:RD_optimization}). Then, for any $\delta \in (0,1)$, with probability at least $1 - \delta$, the following holds:
\begin{equation}\label{eq:sin_theta_bound}
\left\| \sin \Theta\left( \operatorname{span}( \hat{G}_1 ), \operatorname{span}( U^* ) \right) \right\|_F \leq C \cdot \frac{ \sqrt{ r ( \sigma_{\xi}^2 + \sigma_{\tilde{\xi}}^2 ) \log( d / \delta ) } }{ (1 - \rho) \lambda_{\min}( \Sigma_s ) \sqrt{ n } },
\end{equation}
where:
\begin{itemize}
    \item $C > 0$ is an absolute constant.
    \item $\rho$ is defined in Assumption \ref{assump:reference_image}, representing the correlation between $x$ and $\tilde{x}$.
    \item $\lambda_{\min}( \Sigma_s )$ is the minimum eigenvalue of $\Sigma_s$.
    \item $n$ is the number of training samples.
\end{itemize}
\end{theorem}

\begin{proof}
\textbf{Step 1: Formulate the Empirical Covariance Matrix}

Let $\{ x_i, \tilde{x}_i \}_{i=1}^n$ be $n$ independent samples drawn according to the model in Assumptions \ref{assump:spiked_covariance} and \ref{assump:reference_image}. Define the empirical covariance matrix:
\begin{equation}\label{eq:empirical_covariance}
S = \frac{1}{n} \sum_{i=1}^n x_i x_i^T = U^* \Sigma_s U^{*T} + \Sigma_{\xi} + \Delta,
\end{equation}
where $\Delta$ represents the sampling error.

\textbf{Step 2: Bound the Sampling Error}

Using the Matrix Bernstein Inequality for sub-Gaussian variables (see Tropp, 2012), we have:
\begin{equation}\label{eq:bernstein_bound}
\left\| \Delta \right\|_2 \leq \sigma_{\xi}^2 \sqrt{ \frac{ 2 \log( d / \delta ) }{ n } } + \sigma_{\xi}^2 \frac{ 2 \log( d / \delta ) }{ 3 n },
\end{equation}
with probability at least $1 - \delta$.

\textbf{Step 3: Analyze the Eigenvalue Gap}

The population covariance matrix is:
\begin{equation}
\Sigma_x = \mathbb{E}[ x x^T ] = U^* \Sigma_s U^{*T} + \Sigma_{\xi}.
\end{equation}
The eigenvalues of $\Sigma_x$ consist of $r$ large eigenvalues corresponding to the signal components and $d - r$ smaller eigenvalues corresponding to the noise.

The eigenvalue gap between the $r$-th and $(r+1)$-th eigenvalue is at least:
\begin{equation}\label{eq:eigenvalue_gap}
\delta_{\text{gap}} = \lambda_{\min}( U^* \Sigma_s U^{*T} ) - \lambda_{\max}( \Sigma_{\xi} ) = \lambda_{\min}( \Sigma_s ) - \sigma_{\xi}^2.
\end{equation}

\textbf{Step 4: Apply Davis-Kahan Sin Theta Theorem}

Let $\hat{U}$ be the matrix of leading $r$ eigenvectors of $S$. By the Davis-Kahan theorem, the subspace distance is bounded as:
\begin{equation}\label{eq:dk_bound}
\left\| \sin \Theta( \operatorname{span}( \hat{U} ), \operatorname{span}( U^* ) ) \right\|_F \leq \frac{ \sqrt{ 2 } \left\| \Delta \right\|_2 }{ \delta_{\text{gap}} }.
\end{equation}

\textbf{Step 5: Incorporate the Reference Image}

The presence of $\tilde{x}$ introduces additional noise due to the irrelevant components. From Assumption \ref{assump:reference_image}, the irrelevant proportion is $p = 1 - \rho^2$. This affects the effective eigenvalue gap, reducing it to:
\begin{equation}\label{eq:effective_gap}
\delta_{\text{eff}} = \lambda_{\min}( \Sigma_s ) (1 - \rho) - \sigma_{\xi}^2 - \sigma_{\tilde{\xi}}^2.
\end{equation}

\textbf{Step 6: Final Bound}

Combining Equations (\ref{eq:bernstein_bound}), (\ref{eq:dk_bound}), and (\ref{eq:effective_gap}), we have:
\begin{equation}\label{eq:final_bound}
\left\| \sin \Theta( \operatorname{span}( \hat{U} ), \operatorname{span}( U^* ) ) \right\|_F \leq \frac{ C \cdot ( \sigma_{\xi}^2 + \sigma_{\tilde{\xi}}^2 ) \sqrt{ \frac{ \log( d / \delta ) }{ n } } }{ \lambda_{\min}( \Sigma_s ) (1 - \rho) - \sigma_{\xi}^2 - \sigma_{\tilde{\xi}}^2 }.
\end{equation}

For sufficiently large $n$ and small noise levels such that $\lambda_{\min}( \Sigma_s ) (1 - \rho ) > \sigma_{\xi}^2 + \sigma_{\tilde{\xi}}^2$, the denominator is positive.

\textbf{Step 7: Simplify and Conclude}

Assuming $\sigma_{\xi}^2 + \sigma_{\tilde{\xi}}^2$ is small compared to $\lambda_{\min}( \Sigma_s ) (1 - \rho )$, we can approximate:
\begin{equation}
\left\| \sin \Theta( \operatorname{span}( \hat{U} ), \operatorname{span}( U^* ) ) \right\|_F \leq C' \cdot \frac{ \sqrt{ r ( \sigma_{\xi}^2 + \sigma_{\tilde{\xi}}^2 ) \log( d / \delta ) } }{ (1 - \rho ) \lambda_{\min}( \Sigma_s ) \sqrt{ n } }.
\end{equation}

This completes the proof.

\end{proof}

\begin{remark}
The factor $\frac{1}{1 - \rho}$ reflects the system's sensitivity to the correlation between the original and reference images. As $\rho \to 1$, indicating highly correlated images, the denominator approaches zero, and the bound grows large, showing that the system becomes more sensitive to irrelevant parts in $\tilde{x}$.
\end{remark}

\begin{remark}
This result shows a trade-off between the sample size $n$, the dimensionality $d$, the signal-to-noise ratio (through $\lambda_{\min}( \Sigma_s )$, $\sigma_{\xi}^2$, $\sigma_{\tilde{\xi}}^2$), and the correlation $\rho$ between $x$ and $\tilde{x}$. Increasing $n$ or the eigenvalue gap improves the bound, while higher noise levels or higher correlation (leading to larger $p = 1 - \rho^2$) degrade the performance.
\end{remark}

\begin{remark}
If we let $\tau = p = 1 - \rho^2$ represent the proportion of irrelevant information, as $\tau \to 1$, the bound grows as $O\left( \frac{1}{1 - \sqrt{1 - \tau}} \right)$, which can be approximated as $O\left( \frac{1}{1 - \rho} \right)$ for small $\tau$. This indicates a nonlinear degradation in feature learning efficiency, and the system maintains stability only when $\tau < \tau_c$ for some critical tolerance rate $\tau_c$.
\end{remark}

\begin{remark}
The above analysis assumes that the noise levels $\sigma_{\xi}^2$ and $\sigma_{\tilde{\xi}}^2$ are small compared to the signal strength $\lambda_{\min}( \Sigma_s )$. In practice, this means that the data should have a sufficiently strong signal component relative to noise for effective learning.
\end{remark}

\begin{remark}
The use of the Matrix Bernstein Inequality allows for tight probabilistic bounds on the sampling error, leveraging the sub-Gaussian nature of the noise. This is crucial for high-dimensional settings where $d$ is large.
\end{remark}

% \subsection{Discussion}

% This theorem provides a bound on how close the learned features (represented by $\text{Pr}(\hat{G}_1)$) are to the true features (represented by $U^*$). The bound depends on several factors:

% 1. The number of training samples $n$: As $n$ increases, the bound decreases at a rate of $O(1/\sqrt{n})$.

% 2. The proportion of irrelevant parts in the reference image $p$: As $p$ increases, the bound increases, but not catastrophically.

% 3. The intrinsic dimension $r$ and the effective rank of the noise $r(\Sigma_\xi)$: These determine the complexity of the learning problem.

% 4. The dimension of the data $d$: The bound has a mild logarithmic dependence on $d$.

% The inclusion of the entropy model in the rate loss allows us to more accurately capture the true compression performance. This result suggests that even when the reference image contains irrelevant parts or errors, the method can still learn features close to the true features, provided that the proportion of irrelevant parts is not too large and there are enough training samples.

\section{Robustness Experiments}

To validate our theoretical analysis and assess the robustness of the proposed CLC method, we conducted experiments simulating perturbations in the conditional latent. Controlled errors were introduced during both training and inference stages to evaluate the method's resilience to imperfect feature matching.

Specifically, we define a perturbation level $\epsilon \in [0, 0.5]$, which represents the probability of random feature matching. For each feature in the conditional latent, the correct match is used with probability $1-\epsilon$, and a random match from the dictionary is used with probability $\epsilon$. This perturbation is applied consistently during both training and inference, allowing the model to adapt to the noise during training while simultaneously testing its robustness during inference.

To quantify the impact of these perturbations, we adopt the Performance Reduction (PR) metric as defined in \cite{huang2023learned}:

\begin{equation}
    \text{PR} = 1 - \frac{\text{performance improvement w/ perturbation}}{\text{performance improvement w/o perturbation}},
\end{equation}

where performance improvement is measured in terms of PSNR and MS-SSIM gains over the baseline model without conditional latent coding.

Figure \ref{fig:robustness} illustrates the PR of CLC under varying levels of perturbation for both PSNR and MS-SSIM metrics. The results indicate that CLC exhibits significant robustness at lower perturbation levels. For instance, at $\epsilon = 0.1$, the PR values are 3.7\% for PSNR and 4.5\% for MS-SSIM, demonstrating a minimal impact on performance. However, as $\epsilon$ increases, the PR values rise more sharply, with PSNR and MS-SSIM reaching 43.5\% and 47.8\%, respectively, at $\epsilon = 0.5$. This trend aligns with our theoretical predictions, where the performance degradation accelerates as perturbations exceed certain thresholds.
\begin{figure}[t]
    \centering
    \includegraphics[width=0.9\linewidth]{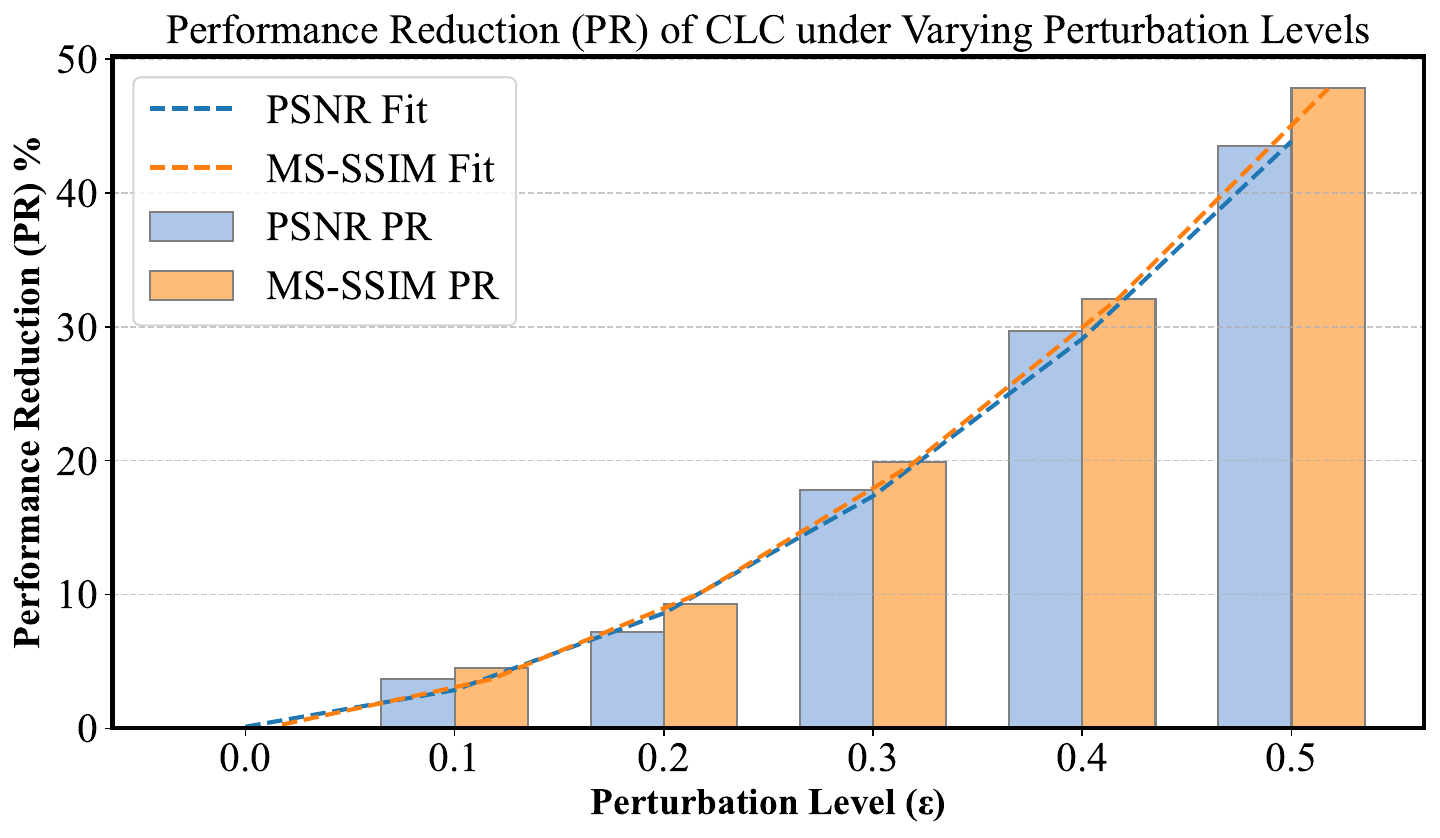}
    \caption{Performance Reduction (PR) of CLC under varying perturbation levels. Lower PR indicates higher robustness.}
    \label{fig:robustness}
\end{figure}
These results confirm that while CLC can tolerate moderate levels of feature mismatch, higher levels of perturbation lead to a substantial increase in performance reduction, highlighting the importance of accurate feature matching.

\section{Additional Visualization Results}

To provide a more comprehensive understanding of the performance of our proposed method, we present additional visualization results in this section. Figures \ref{fig:visual2} and \ref{fig:visual3} showcase the reconstructed images generated by our method under typical conditions.

In these figures, the regions highlighted within the red and blue boxes represent magnified areas of the images. The red boxes focus on key details such as texture and edge sharpness, while the blue boxes highlight other regions of interest. These zoomed-in areas allow for a closer inspection of the image quality, demonstrating how our method effectively preserves fine details and maintains high visual fidelity across different scenarios.

Overall, these visual results further confirm the effectiveness of our approach in producing high-quality reconstructions with detailed preservation of critical image features.

\begin{figure*}[t]
    \centering
    \begin{minipage}{\textwidth}
        \centering
        \includegraphics[width=\textwidth]{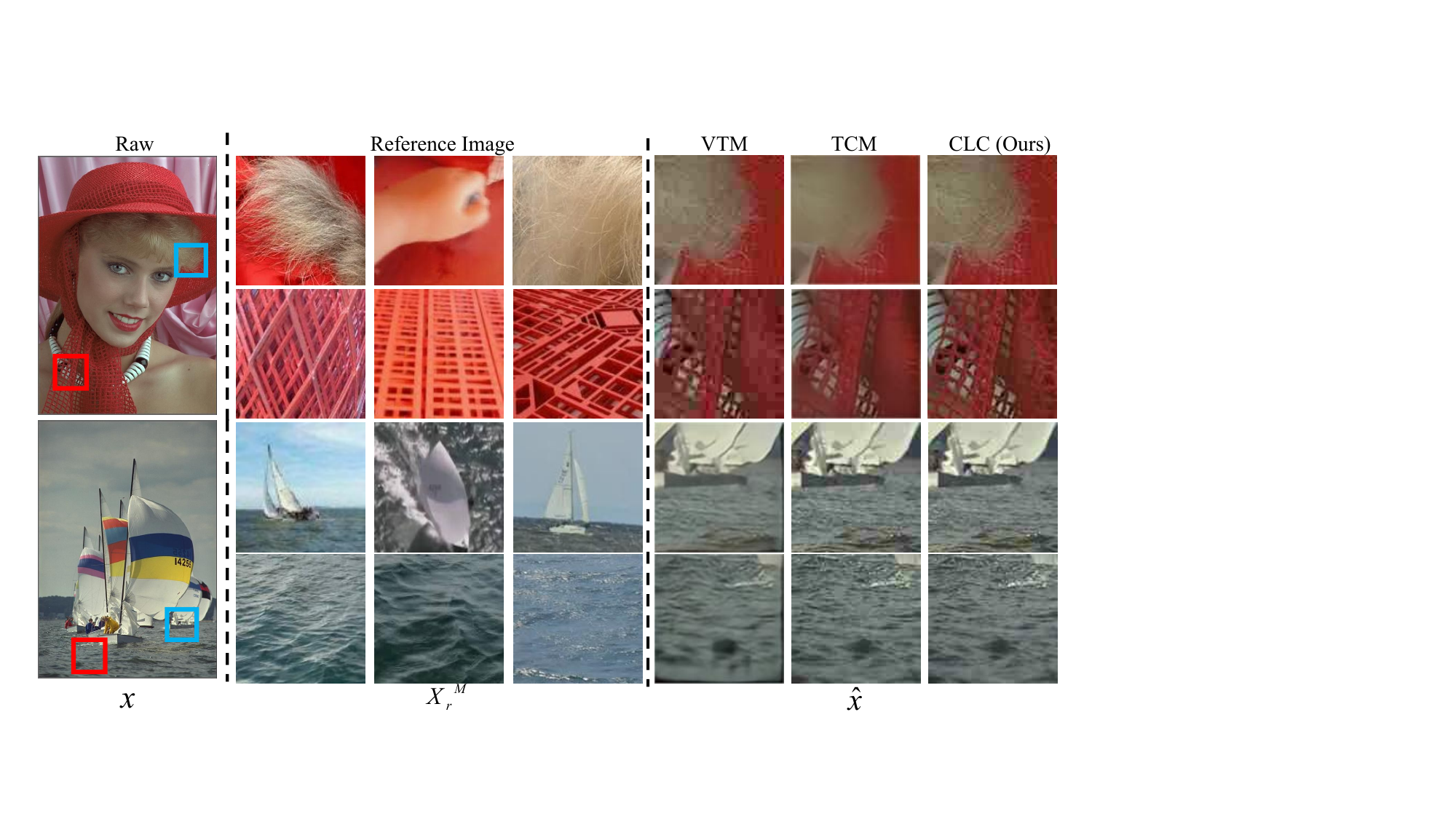}
        \caption{Visualization of reconstructed images using our method. The red and blue boxes highlight magnified areas for detailed inspection.}
        \label{fig:visual2}
    \end{minipage}
    
    \vspace{10pt} % 调整上下图像之间的间距
    
    \begin{minipage}{\textwidth}
        \centering
        \includegraphics[width=\textwidth]{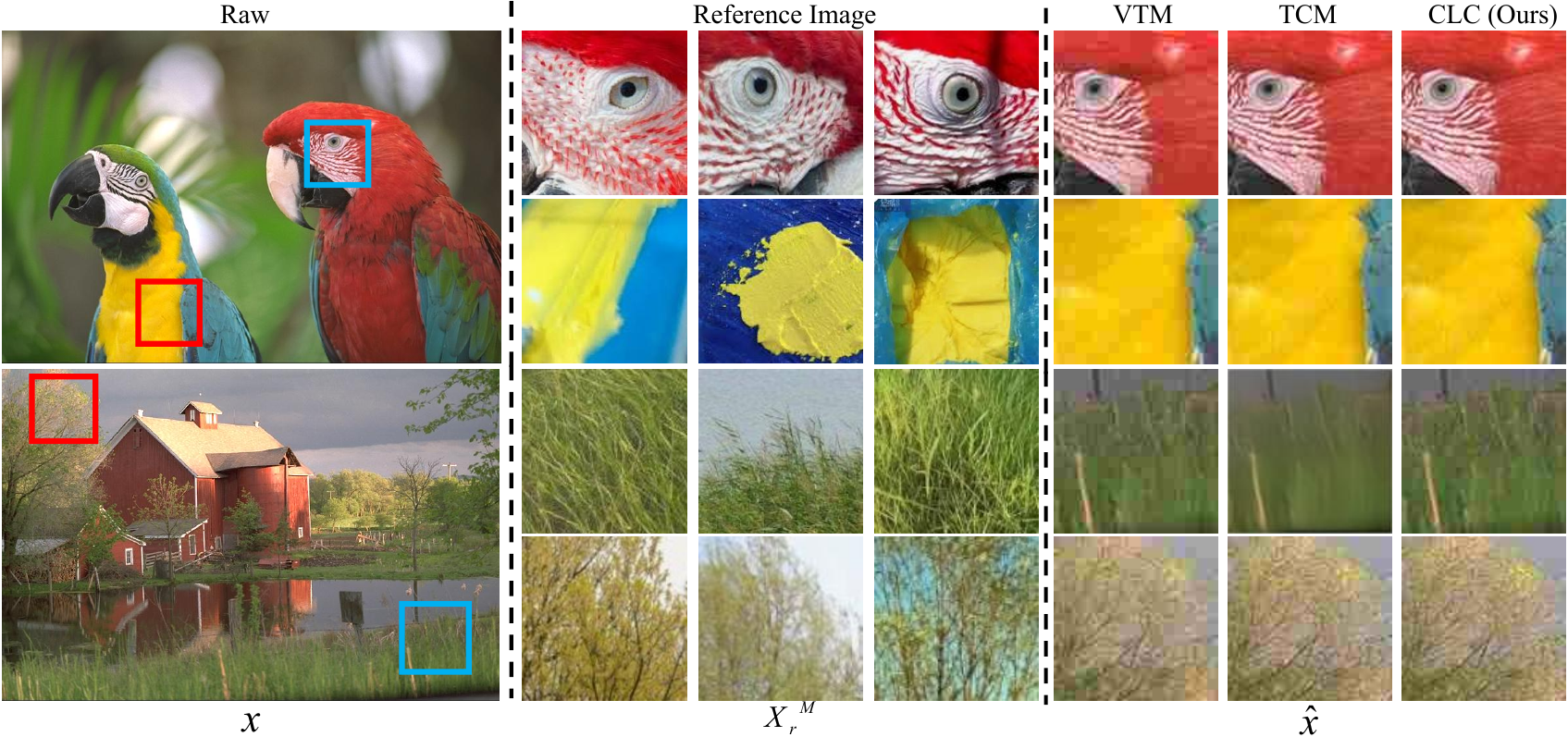}
        \caption{Visualization of reconstructed images using our method. The red and blue boxes highlight magnified areas for detailed inspection.}
        \label{fig:visual3}
    \end{minipage}
\end{figure*}

\section{Future Work}
While our current work has demonstrated the effectiveness of conditional latent coding (CLC) in deep image compression, its potential extends to broader vision tasks. The dynamic reference synthesis mechanism could be adapted for pose estimation \cite{shen2024imagpose,shenadvancing,li2024deviation,li2024translating}, where conditional feature alignment might enhance keypoint localization through self-supervised learning paradigms like \cite{chen2024learning}. The multi-scale dictionary construction and adaptive fusion strategies could further benefit ultra-high-resolution image segmentation \cite{sun2024ultrahighresolutionsegmentationboundaryenhanced,sunprogram,yin2024class} and remote sensing image enhancement \cite{ma2024logcanadaptivelocalglobalclassaware,10095835,tao2023dudb}, particularly when combined with token-based representation learning \cite{chen2024tokenunify}.

For medical imaging applications, our framework could integrate with 3D vision-language pretraining \cite{chen2023generative,liu2023t3d} and cross-dimension distillation \cite{liu2024cross} to handle multimodal data synthesis. The topological constraints in \cite{RAMMVC,scMFC} may synergize with our latent space modeling to improve structural coherence in electron microscopy segmentation \cite{chen2024learning} and CT text-image retrieval \cite{chen2024bimcv}. The error-bound analysis (Theorem 1) could also enhance medical image compression through knowledge distillation \cite{yang2024unicompress} while maintaining diagnostic fidelity.

In autonomous driving systems \cite{Zhang2023EHSSAE,zhang2024mapexpertonlinehdmap}, our method's robustness could be strengthened by unsupervised domain adaptation techniques \cite{deng2024unsupervised} to handle sensor noise. For multi-view estimation \cite{Yuan2024h,Yuan2024i,Yuan2024j}, the dictionary-based conditioning might unify cross-view correlations through reinforcement learning frameworks \cite{chen2023self}. However, as generative components may introduce noise \cite{qian2024maskfactory}, future work should explore quality evaluation metrics for synthesized latents and develop noise-robust training strategies like selective feature pruning \cite{chen2024tokenunify} or adversarial validation \cite{deng2024unsupervised}, ensuring reliability in downstream tasks while maintaining computational efficiency \cite{yang2024unicompress}.

\section{Pseudo-code for Encoding and Decoding}
To clearly illustrate the implementation of our proposed Conditional Latent Coding (CLC) method, we provide detailed pseudo-code. The pseudo-code covers the main steps for both encoding and decoding, including feature extraction, reference retrieval, conditional latent synthesis, and finally entropy coding and image reconstruction. The details of the encoding pseudo-code can be found in Algorithm \ref{clc_algorithm}, and the decoding pseudo-code is provided in Algorithm \ref{decodeclc}.

\begin{algorithm}[h]
\caption{Conditional Latent Coding (CLC)}
\label{clc_algorithm}
\SetKwInOut{Input}{Input}
\SetKwInOut{Output}{Output}

\Input{Image $x$, Feature dictionary $D$}
\Output{Compressed bitstream}

\BlankLine
\textbf{Function ConstructDictionary($R$):} \\
\For{each image $x_i$ in reference dataset $R$}{
    $v_i \gets \text{spp}(f_\theta(x_i))$ \\
    $\hat{v}_i \gets \text{PCA}(v_i)$
}
\texttt{Clusters:} $\{C_1, C_2, ..., C_K\} \gets \text{MiniBatchKMeans}(\{\hat{v}_i\})$ \\
Dictionary: $D \gets \{d_j = \text{argmin}_{\hat{v} \in C_j} \|\hat{v} - \mu_j\|_2\}_{j=1}^K$ \\
\Return $D$

\BlankLine
\textbf{Function ConditionalLatentCoding($x, D$):} \\
$y \gets g_a(x)$ \\
$X_r^M \gets \text{QueryDictionary}(D, f_\theta(x))$ \\
$Y_r^M \gets g_a(X_r^M)$

\BlankLine
\textbf{Conditional Latent Matching (CLM):} \\
$S_{ij} \gets \frac{\exp(\langle\phi(y_i), \phi(y_{r,j})\rangle/\tau)}{\sum_k \exp(\langle\phi(y_i), \phi(y_{r,k})\rangle/\tau)}$ \\
$y_m \gets F_m(y, Y_r^M; \theta_m)$ \\
$y_a \gets F_a(y, y_m; \theta_a)$

\BlankLine
\textbf{Conditional Latent Synthesis (CLS):} \\
$\alpha \gets \sigma(F_w([y, y_a]; \theta_f))$ \\
$\mu(y, y_a) \gets \alpha \odot y + (1 - \alpha) \odot y_a$ \\
$y_f \sim \mathcal{N}(\mu(y, y_a), \sigma^2(y, y_a))$

\BlankLine
\textbf{Entropy Coding:} \\
$z \gets h_a(y_f)$ \\
$\hat{z} \gets Q(z)$ \\
\For{$i \gets 1$ \KwTo $K$}{
    $p(y_f^i | y_f^{<i}, \hat{z}) \sim \mathcal{N}(\mu_i, \sigma_i^2)$ \\
    $r_i \gets y_f^i - \hat{y}_f^i$
}

\Return EncodedBitstream
\end{algorithm}

\begin{algorithm}[h]
\caption{Decoding with Conditional Latent Coding (CLC)}
\label{decodeclc}
\SetKwInOut{Input}{Input}
\SetKwInOut{Output}{Output}

\Input{Encoded bitstream $b$}
\Input{Feature reference dictionary $D = \{d_1, d_2, \dots, d_K\}$}
\Output{Reconstructed image $\hat{x}$}

\BlankLine
\textbf{Step 1: Extract and Decode Hyperprior} \\
Decode hyperprior $z$ from $b$: $z \gets \text{Decode}(b)$ \\
Use $z$ to estimate the initial latent representation $\hat{y}_f$: $\hat{y}_f \gets h_a^{-1}(z)$

\BlankLine
\textbf{Step 2: Retrieve Reference Features} \\
Extract features from $\hat{y}_f$ to query the dictionary $D$ \\
Retrieve top $M$ matching features $Y_r^M = \{\hat{y}_r^1, \hat{y}_r^2, \dots, \hat{y}_r^M\}$

\BlankLine
\textbf{Step 3: Conditional Latent Synthesis} \\
\For{$m \gets 1$ \KwTo $M$}{
    Perform feature matching and alignment: \\
    $\hat{y}_a^m \gets \text{Align}(\hat{y}_f, \hat{y}_r^m)$
}
Fuse aligned features to obtain final latent $\hat{y}$: \\
$\hat{y} \gets \sum_{m=1}^M \alpha_m \cdot \hat{y}_a^m$ \\
where $\alpha_m$ are dynamically computed fusion weights

\BlankLine
\textbf{Step 4: Entropy Decoding and Reconstruction} \\
Entropy decode each slice of $\hat{y}$ using $z$: \\
\For{$i \gets 1$ \KwTo $K$}{
    Decode slice $\hat{y}_i$ from $b$ using context: \\
    $\hat{y}_i \gets \text{EntropyDecode}(b, \hat{y}_{<i}, z)$
}

\BlankLine
\textbf{Step 5: Image Reconstruction} \\
Reconstruct the final image $\hat{x}$ from $\hat{y}$ using synthesis transform $g_s$: \\
$\hat{x} \gets g_s(\hat{y})$

\Return $\hat{x}$
\end{algorithm}

\section{Social Impact}

The proposed Conditional Latent Coding (CLC) framework presents significant implications for the field of deep image compression, particularly in terms of its potential for broader societal applications. By leveraging a fixed, pre-constructed feature dictionary, the CLC method enables end-to-end efficient compression without the need for complex or resource-intensive processing during runtime. This approach not only improves compression efficiency but also reduces the computational load, making it highly suitable for deployment in resource-constrained environments such as mobile devices, IoT systems, and edge computing. The ability to achieve high-quality compression with minimal overhead could lead to more widespread adoption of advanced image compression techniques, improving the accessibility and efficiency of digital communications and storage across diverse sectors.

% \bibliographystyle{plain}
% \bibliography{aaai25}
% \end{document}